\documentclass{article}

\usepackage[final]{neurips_glfrontiers_2022}

\usepackage[utf8]{inputenc} 
\usepackage[T1]{fontenc}    
\usepackage{url}            
\usepackage{booktabs}       
\usepackage{amsfonts}       
\usepackage{nicefrac}       
\usepackage{microtype}      
\usepackage{xcolor}         

\usepackage[colorlinks,
            linkcolor=blue,       
            anchorcolor=blue,  
            citecolor=blue,        
            ]{hyperref}

\usepackage{amsmath}
\usepackage{amssymb}
\usepackage{mathtools}
\usepackage{amsthm}

\theoremstyle{plain}
\newtheorem{theorem}{Theorem}[section]

\theoremstyle{definition}

\theoremstyle{remark}

\usepackage{caption}
\usepackage{pifont}
\usepackage{makecell}
\usepackage{graphicx}
\usepackage{subfigure} 
\usepackage{amsmath}
\usepackage{multirow}
\usepackage{multicol}
\usepackage{enumitem}
\usepackage{ulem}
\usepackage{threeparttable}
\usepackage{wrapfig}
\usepackage{scrextend}
\usepackage{thm-restate}
\usepackage{array,hhline}
\usepackage{amsmath,amsfonts}
\usepackage{graphicx}
\usepackage{tabularborder}
\usepackage{float}
\usepackage{stfloats}
\usepackage{wrapfig}

\newcommand{\mbf}[1]{\mathbf{#1}}
\newcommand{\mbb}[1]{\mathbb{#1}}
\newcommand{\mcal}[1]{\mathcal{#1}}

\def\Done{}

\def\W{\mathbf{W}}
\def\L{\mathbf{L}}
\usepackage{amssymb} 
\def\R{\mathbb{R}}
\def\I{\mathbf{I}}

\def\Q{\mathbf{Q}}

\def\x{\mathbf{x}}
\def\v{\mathbf{v}}

\def\X{\mathbf{X}}

\def\xx{\times}
\def\V{\mathcal{V}}
\def\E{\mathcal{E}}

\def\H{\mbf{H}}
\def\D{\mbf{D}}
\def\W{\mbf{W}}
\def\T{\mbf{T}}

\def\L{\mbf{L}}

\title{A Simple Hypergraph Kernel Convolution based on Discounted Markov Diffusion Process}

\author{
  Fuyang Li\thanks{Equal contribution}
    \\
  SIGS, Tsinghua University\\
  \texttt{lfy20@mails.tsinghua.edu.cn} \\
   \And
   Jiying Zhang$^*$ \\
  SIGS, Tsinghua University\\
   \texttt{zhangjiy20@mails.tsinghua.edu.cn} \\
   \AND
   Xi Xiao\thanks{Corresponding author} \\
  SIGS, Tsinghua University\\
   \texttt{xiaox@sz.tsinghua.edu.cn} \\
   \And
   Bin Zhang \\
   Peng Cheng Laboratory\\
   \texttt{bin.zhang@pcl.ac.cn} \\
   \And
   Dijun Luo \\
   Tencent AI LAB\\
   \texttt{dijunluo@tencent.com} \\
}

\begin{document}

\maketitle

\begin{abstract}
 Kernels on discrete structures evaluate pairwise similarities
between objects which capture semantics and inherent topology information. Existing kernels on discrete structures are only developed by topology information(such as adjacency matrix of graphs), without considering original attributes of objects. 
This paper proposes a two-phase paradigm to aggregate comprehensive information on discrete structures,  leading to a Discount Markov Diffusion Learnable Kernel (DMDLK). Specifically, based on the underlying projection of DMDLK, we design a Simple Hypergraph Kernel Convolution (SHKC) for hidden representation of vertices. SHKC can adjust diffusion steps rather than stacking convolution layers to aggregate information from long-range neighborhoods which prevents oversmoothing issues of existing hypergraph convolutions.Moreover, we utilize the uniform stability bound theorem in transductive learning to analyze critical factors for the effectiveness and generalization ability of SHKC from a theoretical perspective. The experimental results on several benchmark datasets for node classification tasks verified the superior performance of SHKC over state-of-the-art methods.
\end{abstract}

\section{Introduction}
In real-world applications, original data with discrete structures is relatively prevalent. For instance, graphs are commonly used in citation networks, social networks, and protein interaction networks to represent pairwise topology relationships between discrete objects. 
In the community of Graph Neural Networks (GNNs)~\citep{gcn,chen2021diversified,zhang2022fine} and Hypergraph Neural Networks~\citep{hgnn,hypergcn}, researchers focused on designing convolution operators to obtain effective hidden representations for downstream tasks. It is worth addressing that the spectral view and message passing view play important roles in inspiring the design of those convolutions. In another perspective, we start from abstracting a convolution to a  projection $\phi$. Most of the existing convolution networks for the discrete structure like graphs and hypergraphs are essentially derived from some heuristic intuitions and rational principles aiming to obtain an effective representation extraction projection $\phi$. This hidden projection maps objects from original space $\mbf \Omega$ to hidden representations in Euclidean space. 
In some sense, choosing a kernel on $\mbf \Omega\times \mbf\Omega$ is the same as choosing a feature extraction projection for objects in $\mbf\Omega$~\citep{haussler1999convolution}. From this perspective, we need to generate a symmetrical and positive semi-definite kernel matrix  on $\mbf \Omega\times \mbf\Omega$~\citep{kondor2002diffusion} to obtain a rational projection. However, earlier researches pay attention to generating kernels on graphs directly from  the prior topology(i.e. adjacency matrix) without considering the inherent original attributes of vertices. We propose a two-phase paradigm to conduct Discount Markov Diffusion Learnable Kernels(DMDLK) for discrete structures. The respective roles of the two phases are as follows: (i) aggregate topology structure information with original attributes; (ii) aggregate information among channels by learnable parameters. 

If the topology of the original data can be determined a priori, for instance, given the adjacency matrix for a graph, the two-phases paradigm can be naturally adapted to conduct convolution operators for any  discrete structure. However, there are flaws to model discrete data by graphs when original objects inherently meet complex topology structures. Studies have continued to concentrate on utilizing hypergraphs to represent complicated topology structures~\citep{hgnn,zhou2006learning,agarwal2006higher,2019random,ourspaper,zhang2022learnable}. In this paper, we pay attention to hypergraphs which serve as more general discrete structures than simple graphs. Specifically, we modify a generalized transition matrix from~\cite{ourspaper} to 
develop DMDLK for hypergraphs leading to a novel convolution operator for hypergraphs named Simple Hypergraph Kernel Convolution (SHKC).
It can aggregate long-range information and obtain high-level performance by enlarging the diffusion step rather than stacking multi convolution layers. The latter always leads to the well-known oversmoothing issue. To illustrate the effectiveness of SHKC from a theoretical view, we use the tool of uniform stability theorem in transductive learning to analyze the rationality and effect of detailed factors in SHKC. 

Overall, the main contributions highlight as follows:
    (i) We propose a two-phase paradigm to conduct DMDLK for discrete structures;
(ii) By introducing a modified generalized transition matrix of hypergraph into this paradigm, we obtain Simple Hypergraph Kernel Convolution(SHKC) which can avoid stacking multi-layers to aggregated long-range information by simply adjusting the diffusion steps;
    (iii) We utilize uniform stability theorem to explain the promising effectiveness and generalization ability of SHKC. (iv) Empirical results on downstream node classification tasks and object recognition tasks show the advanced performance of SHKC which corresponds well to the theoretical analysis.

\section{Related Works}
\subsection{Kernels for Discrete Structures}
Kernel functions~\citep{scholkopf2002learning,shawe2004kernel,fouss2012experimental} $k:\mbf \Omega \times \mbf \Omega \rightarrow \mathbb{R}$ can directly compute pairwise similarities $k(\mbf x,\mbf x')$ by implicitly constructing a projection $\phi:\mbf \Omega \rightarrow \mcal H_k$ from the original space $\mbf \Omega$ into a high-dimensional and more well-separated Hilbert space $\mcal H_k$.  
Discrete structures, such as strings,graphs,and trees, are dominant in real-world applications like classification and collaborative recommendation~\citep{fouss2012experimental}. Kernels on discrete structures focus on mining pairwise similarity between objects which captures the semantics inherent in discrete structures~\citep{fouss2012experimental}, such as length of the shortest path and the total number of paths between vertices in a graph. Kernels on discrete structures were first proposed by \cite{kondor2002diffusion} and expanded on graphs by \cite{smola2003kernels}.\cite{fouss2006experimental} gives a thorough overview of several kernels on graph and proposes a kernel named regularized commute-time kernel. However, those research focus on capturing similarity between vertices directly from the topology(i.e. adjacency matrix) without considering inherent attributes of vertices. 

\subsection{Diffusion Process on Graphs}
GDC\citep{klicpera2019diffusion} firstly proposes generalized graph diffusion to achieve significant performance improvement across a wide range of tasks. PPNP\citep{ppnp} utilizes personalized PageRank to formulate propagation procedure for GCNs and concluded a fast approximation version APPNP. The diffusion process defined on APPNP can be found in \citet{li2020quadratic} which is differ from our defined discounted Markov diffusion process in Section \ref{twophase}. We start by utilizing discounted average visiting rate to define the diffusion process rather than the recursive propagation in APPNP leading to a slight difference in the coefficients of  different hop neighborhoods. Furthermore, we focus on applying the thought of diffusion process for more complex discrete data, like hypergraphs.

\subsection{Hypergraphs Convolution Networks}
Hypergraphs whose edges contain more than two vertices can be seen as more generalized discrete structures  than graphs. Hypergraph have shown its promising ability to model more complex topology information than graphs~\citep{zhou2006learning,hgnn}. \cite{hgnn} firstly introduces hypergraphs into deep learning community motivated by capturing multi-modal topology information through vertex-hyperedge incident matrix. \cite{hypergcn} introduces a non-linear Laplacian Matrix for hypergraphs deducing a more expressive convolution. \cite{dong2020hnhn} proposes a message -passing based model with a two steps message propagation between vertex and hyperedge. \cite{HyperGAT} involves attention mechanisms in hypergraphs for the task of text classification. Moreover, thanks to the intensive research of Markov process on hypergraphs in~\citep{zhou2006learning,2019random,carletti2021random,ourspaper}, we immediately stitch the generalized hypergraph transition matrix~\citep{ourspaper} with the two-phases paradigm to conduct our effective SHKC. 
\section{Preliminaries}

\paragraph{\textbf{Notions.}} In this paper, we use a boldface capital letter
$\mbf A\in\R^{N\times M}$ to denote an $N \times M$ matrix and use $A_{ij}$ or $A(i,j)$
to denote its ${ij}$-th entry. We use the boldface letter $\x$ or $\vec{\theta}$
 to indicate the column vector, where $x_i$ or $\vec{\theta}_i$ is
the $i$-th entry of $\x$ or $\vec{\theta}$. The word "vector" always denotes a column vector in this paper. Thus, we use the transposition of the column vector $\x^{\top}$ or $(\vec{\theta})^{\top}$ to denote the row vector. For vectors with subscript like $\x_k$ and $\vec{\theta}_k$ mean $k$-th column of the matrix where $x_{ki}$ and $\vec{\theta}_{ki}$ also mean their $i$-th component. 
Let $\I\in\R^{N\xx N}$ denote the  identity matrix,
and $\mbf e_i$ represents the vector where the $i$-th component is 1 and otherwise 0.

\subsection{Kernels on Discrete Structures} 
Given an original space $\mbf \Omega$ with $\vert\mbf \Omega \vert=N$ discrete objects, the kernel function $k$ can be uniquely represented by a $N\times N$ matrix called Kernel Matrix. The Matrix should generally be symmetrical and positive semi-definite~\citep{haussler1999convolution}. 
A finite vertex set $\mcal V = \{\mbf v_1,\cdots,\mbf v_N\}$ denotes the original space $\mbf \Omega$ and an implicit $\phi(\mbf v_i)$  maps the vertex $\mbf v_i$ to a Hilbert space $\mcal H_k$. In fact, $\phi(\mbf v_i)$ denotes a hidden representation for $\mbf v_i$ in $\mcal H_k$. Thus,  We call $\mcal H_k$ "hidden representation space" according to its specific meaning.
Give original attributes $\mbf X = \{\mbf x_1,\cdots,\mbf x_N\}^{\top} \in \mbb R^{N\times d}$ whose $i$-th row $\mbf x_i^{\top} \in \mbb R^{1\times d}$ denotes a $d$-dimension attribute on the $i$-th vertex $\mbf v_i$. 
Define $\mbf K$ as the kernel matrix where $K_{ij} = k(\mbf v_i, \mbf v_j)$ is the pairwise similarity through a kernel function $k$. 
\paragraph{An Instinctive Kernel on Discrete Structures.}
An effective kernel should capture semantics information inside the discrete structure induced by $\phi$, which cannot be trivially represented by the original attributes as we discussed below.
A most instinctive idea for the design of the kernel is that define $\mbf K_{ori} = \mbf X \mbf X^{\top}$ directly from the original attributes where $K_{ori}(i,j)=k(\mbf v_i, \mbf v_j)=\mbf x_i^{\top}\mbf x_j$. Here the underlying $\phi$ from $\mbf v_i$ to its hidden representation denotes as $\phi(\mbf v_i) = \mbf x_i =\mbf X^{\top}\mbf e_i$. 
Obviously, $\phi$ only captures original attributes without neither structure information nor information aggregated among different channels of original attributes.
It is significant to address that there is a distance in the hidden representation space associated with the kernel: 
$
    (d(\mbf v_i,\mbf v_j))^2 = \Vert\phi(\mbf v_i)-\phi(\mbf v_j) \Vert_2^2
    = (\mbf e_i -\mbf e_j)^{\top}\mbf K_{ori} (\mbf e_i -\mbf e_j)
$.
This indicates meaningful distances defined on hidden representation space can be associated with effective kernels. 
There has been various kernels on graphs aiming at extracting semantics information inside the graphs~\citep{fouss2012experimental}. However, these kernels are developed merely from adjacency matrix of graphs without considering information from original attributes of vertices which means those kernels can not be trivially used in GNNs. From this perspective, we propose a two-phase paradigm for effective kernels and its induced convolutional operators.
\section{Two Phases for Effective Kernels on Discrete Structure}
\label{twophase}
\subsection{Phase I: Topology Information Aggregation with Original Attributes}
In the first phase, we mainly consider effective approaches to excavate structure information from the underlying topology of discrete structures. We start from the perspective of defining a specific Markov Diffusion Process on discrete structures. Note that this idea is directly suitable for hypergraphs as the random walk on hypergraphs has been intensively researched in \citep{2019random,carletti2021random,ourspaper, chen2022preventing}. 

\noindent\textbf{Discounted Average Visiting Rate.}
For a Markov process on a finite state space $\{s_i,\cdots,s_N\}$, we denote $\mbf T$ as the probability transition matrix where $T_{ik} = \Pr\{s(t+1)=s_k|s(t) = s_i\}$. Note that a $\tau$-step transition from state $i$ to $k$ can be formulated as $\Pr\{s(t+\tau)=s_k|s(t)=s_i\}=(\mbf T^{\tau})_{ik}$. We define the \textit{discounted average visiting rate} as:
\begin{align}
\label{eq:avevis}
    \bar{v}_{ik}(t) = \frac{1}{t}\sum_{\tau=1}^{t}\alpha^{\tau}\Pr\{s(\tau)=s_k|s(0)=s_i\},t\in \mathbb{Z}^{++}
\end{align}
The discounted average visiting rate is considered as the probability of a random walker starting from $s_i$ standing in $s_k$ during $t$ steps diffusion.
It is modified by introducing a scalar discounted factor from the average visiting rate~\citep{fouss2006experimental} to alleviate long-range information aggregation  which is motivated by a thought that more global information aggregation makes undistinguished among vertices.

\noindent\textbf{Discounted Markov Diffusion Process.}
Recall that we have the original attributes $\mbf X = \{\mbf x_1,\cdots,\mbf x_N\}^{\top}$ where $\mbf x_i\in \mbb R^d$ is a $d$-dimension vector which can be seen as $d$ channels signal on $\mbf v_i$.  We define a $t$-step process of vertex signal passing among the discrete structure named the \textit{discounted Markov diffusion process with original attributes} which follows the formulation as:
\begin{align}
\label{eq:y1}
    \mbf y_{k}^{(t)} = \beta\sum_{i=1}^{N}\bar{v}_{ik}(t)\mbf x_i + (1-\beta)\mbf x_k 
\end{align}
This equation depicts a process where original attributes are aggregated to the vertex $\mbf v_k$ from all other vertices during a $t$-step diffusion. The diffusion process captures the topology information underlying the discrete structure through the associated transition matrix. We call this phase as \textit{topology information aggregation with original attributes}.
Then we define the \textit{$t$-step Discounted Markov Diffusion distance with original attributes} $d^{(t)}_{M}$ between $\mbf v_i, \mbf v_j$ as:
\begin{align*}
    d^{(t)}_{M}(\mbf v_i,\mbf v_j) &= \Vert \mbf y_i^{(t)}-\mbf y_j^{(t)} \Vert_2
    = \Vert\mbf X^{\top}(\beta\mbf Z(t) + (1-\beta)\mbf I)(\mbf e_i - \mbf e_j) \Vert_2 \\ &= 
     [(\mbf e_i - \mbf e_j)^{\top}\mbf K_{M}(t) (\mbf e_i - \mbf e_j)]^{\frac{1}{2}}
\end{align*}
where $\mbf K_M(t) = (\beta\mbf Z(t) + (1-\beta)\mbf I)^{\top}\mbf X \mbf X^{\top}(\beta\mbf Z(t) + (1-\beta)\mbf I)$ denotes the \textit{Discounted Markov Diffusion Kernel with original attributes} and $\mbf Z(t) = \frac{1}{t}\sum_{\tau=1}^{t}\alpha^{\tau}\mbf T^{\tau} \label{eq:Z(t)}$. The underlying projection behind the kernel denotes $\phi^{(t)}(\mbf v_i)=\mbf e_i^{\top}(\beta\mbf Z(t) + (1-\beta)\mbf I)^{\top}\mbf X$. It is easy to see that $\mbf K_M(t)$ is symmetrical and positive semi-definite which means $\mbf K_M(t)$ is indeed a kernel. However, the underlying projection still has flaws to represent comprehensive vertex representation. The projection failed to involve information interactions between channels of original attributes which leads to the second phase to remedy for the flaws.
\subsection{Phase II: Channels Aggregation }
From Eq.~\eqref{eq:y1}, it is obvious that the vertex representation of $\mbf v_k$ derived from $\mbf K_M$ only captures information within the same channel. Specifically, for channel $c$:~$y_{kc}^{(t)} = \beta\sum_{i=1}^{N}\bar{v}_{ik}(t) x_{ic} + (1-\beta) x_{kc}$.
Thus, in the second phase, we consider aggregating among channels to find a more comprehensive feature space $\mcal H_{k}$ .

\noindent\textbf{Discounted Markov Diffusion Learnable Kernel~(DMDLK).} 
Define $\mbf \Theta =$ $ \{\Vec{\theta}_1,\cdots,\Vec{\theta}_{M}\}\in\mbb R^{d\times M}$ where $\Vec{ \theta}_{m} = \{\theta_{1m},\cdots,\theta_{dm}\}^{\top}\in \mbb R^{d\times 1}$ denotes the $m$-th set of weights concerning all channels and $\theta_{cm}$  specifically denotes the weight of the $c$-th channel in the aggregation. Thus we can aggregate information among channels by:
\begin{align*}
    \tilde{y}^{(t)}_{k}(\vec{\theta}_m) &=\sum_{c=1}^{d}\left( \sum_{i=1}^{N}\beta\bar{v}_{ik}(t) x_{ic}\right)\theta_{cm} + (1-\beta) x_{kc} \theta_{cm}
\end{align*}
Here $\tilde{y}^{(t)}_{k}(\vec{\theta}_m)$ captures both topology information and information aggregated among channels by $\vec{\theta}_m$ for $\mbf v_k$.
Let $\tilde{\mbf y}_{k}^{(t)}(\mbf \Theta) =\{\tilde{y}^{(t)}_{k}(\vec{\theta}_1),\cdots,\tilde{y}^{(t)}_{k}(\vec{\theta}_M)\}^{\top}$.
Naturally, we can define a new distance between two vertices: 
\begin{align*}
 (\tilde{d}^{(t)}(\mbf v_i,\mbf v_j))^2 = \Vert\tilde{\mbf y}_i^{(t)}(\mbf \Theta) - \tilde{\mbf y}_j^{(t)}(\mbf \Theta)\Vert_2^2\nonumber
    =(\mbf e_i-\mbf e_j)^{\top}\mbf K_{M}^{\mbf \Theta}(t) (\mbf e_i-\mbf e_j)
\end{align*}

We define the \textit{Discounted Markov Diffusion Learnable Kernel} as:
$$\mbf K_{M}^{\mbf \Theta}(t)=(\beta\mbf Z(t) + (1-\beta)\mbf I)^{\top}\mbf X\mbf \Theta\mbf\Theta^{\top} \mbf X^{\top}(\beta\mbf Z(t) + (1-\beta)\mbf I)$$
It is also easy to see that $\mbf K_{M}^{\mbf \Theta}(t)$ is symmetrical and positive semi-definite. The underlying projection from the original space $\mbf\Omega$ to the new hidden representation space denotes as 
\begin{align}
\label{eq:underly}
\tilde{\phi}^{(t)}_{\mbf \Theta}(\mbf v_i)=\mbf e_i^{\top}(\beta\mbf Z(t) + (1-\beta)\mbf I)^{\top}\mbf X\mbf\Theta
\end{align}
Totally, the new hidden representation  $\tilde{\phi}^{(t)}_{\mbf \Theta}(\mbf v_i)$ of $\mbf v_k$ can be described as the two phases: (i) original signals on all vertices propagate to $\mbf v_k$ concerning the discounted Markov diffusion process which aggregates the structure information underlying the hypergraph. (ii) $\mbf \Theta$ is introduced to aggregate information among different channels of the representation after the first phase. In the next part, we would directly expand $\mbf K_{M}^{\mbf \Theta}(t)$ to an effective convolutional operator for hypergraphs.

\section{A Specific Instance: Simple Hypergraph Kernel Convolution Networks}
\label{sec:shkc}

To give a specific formulation for a DMDLK, we should identify a specific topology for the discrete structure. In this paper, we make effort for hypergraphs as it has a wider concept and is capable of engaging more sophisticated topology information than graphs.

\subsection{Probability Transition Matrix on Hypergraph}
Assume a hypergraph defined as $\mathcal{H}(\V,\E,\W,$ $\Q_1,\Q_2)$. Here the $\V$ is a finite vertex set with $N$ vertices and $\E$ denotes a hyperedge set where each hyperedge $e \in \E$ can contain more than 2 vertices.
$\mbf W\in\mbb R^{|\mcal E|\times|\mcal E|}$ is a diagonal matrix whose diagonal  denotes the prior weights of hyperedges. $\mbf Q_1$ and $\mbf Q_2 \in \mbb R^{|\V|\times|\E|}$ denote two different edge-dependent vertex weights matrices. $Q_i(v,e)$ denotes the weight of vertex $v$ depending on an incident hyperedge $e$ which means the vertex $v$ contributes different weights to different hyperedges. Note if $e$ is not linked to $v$, $Q_i(v,e)=0$.

Then, we introduce a generalized random walk defined on hypergraphs from \citep{ourspaper} where the probability of the random walk from vertex $u$ to vertex $v$ with a two-step manner is:
\begin{align}
\label{randomwalk}
\small
     \Pr( u,v ) =\sum_e \left(\frac{w(e)Q_1( u,e)\delta(e)\rho(\delta(e))}{d( u)}\right)\left(\frac{Q_2(v,e)}{\delta(e)}\right)
\end{align}
Here, $\delta(e)=\sum_v Q_2(v,e)$ is defined as the degree of hyperedge $e$ and
$d(u) = \sum_e w(e)Q_1(u,e)\delta(e)\rho(\delta(e))$ is defined as the degree of vertex $u$. $\rho(\cdot)$ denotes a selectable function that determines the effect of $\delta(e)$ to $d(u)$. $\mbf Q_1$ and $\mbf Q_2$ play their roles of representing fine-grained topology information in the generalized random walk. In practice, we choose $\mbf Q_1=\mbf Q_2 =\mbf Q$ with a rational assumption that the vertex weights depending on hyperedge keep constant in the two-step random walk. Then we choose to use a modified symmetrical form of $\mbf P$ to define the probability transition matrix $\mbf T$ on hypergraph:
\begin{align}
\label{eq:T}
    \mbf T = \D_{\mcal V}^{-1/2}\mbf Q \mbf W\rho(\D_{\mcal E})\mbf Q^{\top}\D_{\mcal V}^{-1/2}
\end{align}

\noindent\textbf{Hyperedges Containing Isolated Vertex.} 
However, as we analyze the effectiveness of $\tilde{\phi}^{(t)}_{\mbf \Theta}(\mbf v_i)$ from a transductive learning perspective in Section \ref{analysis}, we find it is important to bound the $l_1$-norm of $\mbf T$ to derive a bounded generalization gap between training error and testing error.
We find it derive from a flaw of the defined random walk in Eq.~\eqref{randomwalk} where those hyperedges containing only one vertex(isolated vertex) have not been taken into account. 
Unfortunately, when there are isolated vertices in a hypergraph, the $l_1$-norm of $\mbf T$ defined above can not be bounded, leading to an inadequate performance in downstream classification tasks. This also provides a new view to explain the experiments of \cite{UniGNN} and \cite{ourspaper} w.r.t self-loops. 
We find it derive from a flaw of the defined random walk in Eq.~\eqref{randomwalk} where those hyperedges containing only one vertex(isolated vertex) have not been taken into account. This can be amended by modifying $\mbf Q$: $\Tilde{\mbf Q} = \text{concate}(\mbf Q,\{\mbf e_{i_1},\cdots,\mbf e_{i_k}\})$ if there is an isolated vertex subset $\{\mbf v_{i_1},\cdots,\mbf v_{i_k}\}\subset \mcal V$. Here $\mbf e_{i_k}\in \mbb R^{|\mcal V|} $ denotes a vector whose $i_k$-th component is 1 otherwise 0. However, if the number of isolated vertices is large, this modification leads to extra expensive computation costs due to the added dimension of $\mbf Q$. In fact, this modification is equal to set $T_{ii}=1$ if the original $T_{ii}=0$.  
Thus, we take an approximate method by modifying $\mbf T$ directly to define \textit{the generalized transition matrix on hypergraphs} : 
\begin{align}
\label{eq:tansi}
\Tilde{\mbf T} = \Tilde{\D}_{\mcal V}^{-1/2}(\mbf Q \mbf W\rho(\D_{\mcal E})\mbf Q^{\top} + \I )\Tilde{\D}_{\mcal V}^{-1/2} 
\end{align}
where $\Tilde{D}_{\mcal V}(i,i) =\Tilde{d}(v_i) = 1 + \sum_e w(e)Q(u,e)\delta(e)\rho(\delta(e)) $. This modification is well-known as \textit{the renormalization trick} proposed in GCNs~\citep{kipf2016semi}. We hope to describe why this trick works from this perspective. Finally, we give a proposition below to bound $l_1$-norm of $\Tilde{\mbf T}$. 

\begin{restatable}[Bound $l_1$ norm of $\Tilde{\mbf T}$]{proposition}{restatProone}
\label{Th:Tl1}
For any hyperedge $e\in \mcal E$, let $n_e$ be the number of vertices incident to $e$;For any vertex $v\in \V$, let $n_v$ be the number of hyperedges incident to $v$.
Assume that $\mbf Q$ is a normalized edge-dependent vertex weight matrix ($Q(v,e)\in[0,1]$), $\mbf W$ is a normalized prior hyperedge weights diagonal matrix, $n_{e} \leq E$ for any hyperedge $e$,$n_{v} \leq D$ for any vertex $v$ ,and $\rho_{max}$ is the maximum value of $\rho(x)$ when $x\in[0,E]$. Let $\Tilde{\mbf T}$ be the generalized transition matrix for a hypergraph defined above. Then, $\Vert \Tilde{\mbf T}\Vert_1 \leq \sqrt{1+\rho_{max}ED}$.
\end{restatable}
The proof is referred to in Appendix \ref{proof:p1}. During the analysis of the effectiveness of $\tilde{\mbf y}_k^{(t)}(\mbf\Theta )$ in Section \ref{analysis},
 this proposition plays a critical role to derive a bounded generalization gap between training error and testing error.

\subsection{Simple Hypergraph Kernel Convolution (SHKC)}
 Take the transition matrix $\tilde{\T}$ into the projection $\tilde{\phi}^{(t)}_{\mbf \Theta}(\mbf v_i)$ underlying $\mbf K^{\mbf \Theta}_M(t)$ in Eq.~\eqref{eq:underly}. We have the hidden representation on all vertices as:
 $   \mbf H^{(t)} = (\beta\sum_{\tau=1}^{t}\frac{\alpha^{\tau}}{t} \Tilde{\mbf T}^{\tau} + (1-\beta)\mbf{I})\X\mbf\Theta$.
Recall that $t$ is the diffusion step. $\alpha$ is the introduced discounted factor to weaken long-range global information through the diffusion process and $\beta$ plays a role in balancing information aggregated to the vertex through the diffusion process and original signals in the vertex itself.  By adding an activation function $\psi(\cdot)$ to make up nonlinearity of the underlying projection, we define the \textit{Simple Hypergraph kernel Convolution} as follows:
\begin{align}
\label{eq:shkc}
\small
    \mbf H^{(t)} = \psi\left(\left(\beta\sum_{\tau=1}^{t}\frac{\alpha^{\tau}}{t} \Tilde{\mbf T}^{\tau} + (1-\beta)\mbf{I}\right)\X\mbf\Theta \right)
\end{align}
 where $\mbf\Theta$ is the parameter of the filter to be learned during training. SHKC could gain stronger local and weaker global information, thereby improving the expressive power in a long diffusion step. Thus, SHKC can choose a long step $t$ to replace stacking multi convolution layers for long-range information aggregation. The latter always leads to an oversmoothing issue in GNNs.
Actually, SHKC is a spatial-based model while we can analyze it~(detailed in Appendix \ref{app:spectral_ana}) from a spectral-based view  which may lead to some intrinsic connections between the kernel-based convolution and spectral-based convolution.
\subsection{A Transductive Learning Perspective for Generalization Analysis for SHKC}
\label{analysis}

 For simplicity, we give the theoretical analysis of SHKC in the setting of binary classification. Let $\mbf H^{(L)}$ be the hidden  representation of all vertices with the $L$-steps discounted Markov diffusion process. The $i$-th row $\mbf h^{(L)}_i$ in $\mbf H^{(L)}$ denotes the hidden representation of $\mbf v_i$: $\mbf h^{(L)}_i = \psi(\tilde{\phi}^{(L)}_{\mbf \Theta}(\mbf v_i))$. Let $\vec{\omega}$ be the weight of a downstream classifier $f$. Thus, the prediction label of vertex $\mbf v_i$ denotes as $f(\mbf h^{(L)}_i) = \sigma(\mbf h^{(L)}_i\vec{\omega})$ where $\sigma$ is the sigmoid function. 
 
 We start from introducing the transductive uniform stability bound(USB) theorem from \cite{el2006stable}(detailed in Appendix \ref{app:usb}). The USB theorem shows that the generalization gap depends on the uniform stability $\mu$ of $\mathcal{M}$. A most recent work \citep{cong2021provable} decompose $\mu$ into three parts: Lipschitz constant $L_{\mathcal{M}}$, upper bound on gradient $G_{\mathcal{M}}$, and smoothness constant $S_{\mathcal{M}}$ of the learning model $\mathcal{M}$, which is formulated as $\mu = \frac{2\eta L_{\mathcal{M}} G_\mathcal{M}}{m}\sum_{t=1}^{T}(1+\eta S_\mathcal{M})^{t-1}$ where $T$ is training steps and $\eta$ is the learning rate. We follow this approach to evaluate the effectiveness of $\tilde{\phi}^{(L)}_{\mbf \Theta}$ which is described by the generalization gap of SHKC.
Then the main effort to analyze the generalization gap boils down to the bound of $L_{\mathcal{M}},G_{\mathcal{M}},S_{\mathcal{M}}$ for SHKC. 
We firstly start from three Lemmas shown in Appendix \ref{app:main_lemma}.
Then, we get our main results:
\begin{restatable}[Main Results For SHKC]{theorem}{restartTheoremOne}\label{th:main}
Assume that (i) for any vertex $i$, its norm of the feature vector is bounded by a constant $\Vert \mbf x_i \Vert_2 \leq C_x$; (ii) the norm of learnable parameters of SHKC is bounded as $\Vert\mbf \Theta \Vert_2 \leq C_{\mbf\Theta}$; (iii) the norm of the weight of the classifier is bounded as: $\Vert\vec{\omega}\Vert_2\leq 1$.
Let $\alpha,L,\beta$ be the hyper-parameters  defined in SHKC during training  and $\eta$ be the learning rate. Then the model SHKC is $\mu_{SHKC}$ uniform stable with $\mu_{SHKC}=\frac{2\eta L_{\mathcal{M}} G_\mathcal{M}}{m}\sum_{t=1}^{T}(1+\eta S_\mathcal{M})^{t-1}$
where:
\begin{align*}
\small
    &L_{\mathcal{M}}= C_x C_{\alpha\beta L}\max\{1,C_{\Theta}\};\quad G_{\mathcal{M}}=C_x C_{\alpha\beta L}(1+C_{\Theta}); \\ &S_{\mathcal{M}}=C_x^2 C_{\alpha\beta L}^2\max\{1,C_{\Theta}\}^2 + C_x^2 C_{\alpha\beta L}^2 C_{\Theta} + C_x C_{\alpha\beta L}
\end{align*}
Here $C_{\alpha\beta L} = \frac{\beta}{L}\sum_{l=1}^{L}(\alpha d_T)^l + (1-\beta)$ and $d_T = \sqrt{1+\rho_{max}ED}$ is the $l_1$ norm of the generalized transition matrix for a hypergraph in Proposition \ref{Th:Tl1}.
\end{restatable}
Proof sees in Appendix~\ref{th:main_proof}.
The key of the three components uniquely corresponded to SHKC is $C_{\alpha\beta L}$. We conclude that:  (i) In $C_{\alpha\beta L}$,  it is easy to see that $\alpha$ discounts $d_T^l$ when $l$ is large which corresponds to our intuition to weaken the effect of global information on vertex representation. (ii) The defined discounted average visiting rate in Eq.~\eqref{eq:avevis} 
leads to $\frac{1}{L}$ in $C_{\alpha\beta L}$ which relatively limits the enlargement of $C_{\alpha\beta L}$ as $L$ increases. This keeps relative balanced weights to information aggregated from different diffusion steps. (iii) $\beta$ is introduced to balance the $C_{\alpha\beta L}$ between $\frac{1}{L}\sum_{l=1}^{L}(\alpha d_T)^l$ and $1$. This corresponds to the intuition of balancing the effect between aggregated information and original feature to the hidden representation of vertices. 
Furthermore, this theorem reveals that SHKC can tighten the generalization bound by adjusting $\alpha$,$\beta$ and $L$.

\paragraph{Contributions of DMDLK and SHKC.}
Firstly, it is significant to discuss our contribution over ~\citet{zhang2022hypergraph}. They propose a framework for transforming existing GNNs to HyperGNNs based on a equivalency condition. We utilize a core technique from them which is involving fine-grained edge-vertex topology information to construct comprehensive probability transition matrix. However, we start from a two-phase kernel-based perspective to conduct the convolutional operator for hypergraph rather than the spectral perspective in \citet{zhang2022hypergraph}. This means it is possible to conduct convolutional operator for wider hypergraphs which are not satisfied the equivalency conditions in the previous paper. Furthermore, we show that the renormalization trick helps to bound $l_1$-norm of $\tilde{\mbf T}$ for hypergraphs containing isolated vertex
which provides a theoretical perspective to explain the promising trick.
Secondly, previous studies have shown that graph diffusion leads to significant performance improvements for GNNs\citep{klicpera2019diffusion}. We concentrate on explaining that increasing diffusion step in DMDLK can avoid stacking convolutional layers to alleviate over-smoothing issues and showing that diffusion thoughts can be naturally applied to broader discrete data with complex topology information than graphs. Overall, we conduct the kernel-based two-phase paradigm aiming at designing simple and comprehensive convolutional operator for representing wider discrete data. We also provide a theoretical perspective to explain how tricks(i.e. renormalization) and critical parameters(i.e. diffusion step, discount factor $\alpha$ and balance factor $\beta$ ) works in SHKC.

\section{Experiments}
\subsection{Citation Network Classification}
This is a semi-supervised node classification task.
The datasets we use are hypergraph benchmarks constructed by \citet{hypergcn}(See Appendix Table~\ref{tab:dataset_hypergcn}). We adopt the same public datasets\footnote{{\url{https://github.com/malllabiisc/HyperGCN}}} and train-test splits in \citet{hypergcn}. Note these datasets 
satisfy $\Q_1=\Q_2=\H$.
For baselines, we involve MLP with explicit Hypergraph Laplacian Regularization (MLP+HLR), HNHN, HyperSAGE, HGAT, UniGNN, HGNN and HyperGCN.
\begin{table*}[th]
\def\p{$\pm$} 
\centering
\setlength\tabcolsep{8pt} 
\caption{Summary of classification accuracy(\%) results. We report the average test accuracy and its standard deviation over 10 train-test splits. The number in parentheses corresponds to the number of diffusion step of SHKC. (OOM: out of memory)}
\scalebox{0.67}{
    \begin{tabular}{l|l|c|c|c|c|c}
        \toprule 
              \multicolumn{1}{c}{Dataset}& \multicolumn{1}{c}{Architecture} &\multicolumn{1}{c}{\begin{tabular}[c]{@{}c@{}}Cora\\ (co-authorship)\end{tabular}}  & \multicolumn{1}{c}{\begin{tabular}[c]{@{}c@{}}DBLP\\ (co-authorship)\end{tabular}} & \multicolumn{1}{c}{\begin{tabular}[c]{@{}c@{}}Cora\\ (co-citation)\end{tabular}} & \multicolumn{1}{c}{\begin{tabular}[c]{@{}c@{}}Pubmed\\ (co-citation)\end{tabular}} & \multicolumn{1}{c}{\begin{tabular}[c]{@{}c@{}}Citeseer\\ (co-citation)\end{tabular}} \\
        \midrule
             MLP+HLR & - & 59.8\p4.7 & 63.6\p4.7 & 61.0\p4.1 & 64.7\p3.1 & 56.1\p2.6 \\
             FastHyperGCN~\cite{hypergcn} & spectral-based & 61.1\p8.2 & 68.1\p9.6 & 61.3\p10.3 & 65.7\p11.1 & 56.2\p8.1 \\
             HyperGCN~\cite{hypergcn} & spectral-based & 63.9\p7.3 & 70.9\p8.3 & 62.5\p9.7& 68.3\p9.5 & 57.3\p7.3 \\
            HGNN~\citep{hgnn} & spectral-based & 63.2\p3.1 & 68.1\p9.6 & 70.9\p2.9& 66.8\p3.7 & 56.7\p3.8 \\
             HNHN \citep{dong2020hnhn}& message-passing & 64.0\p 2.4 & 84.4\p 0.3& 41.6\p 3.1 & 41.9\p4.7 & 33.6\p 2.1  \\
             HGAT~\citep{HyperGAT} & message-passing & 65.4$\pm$1.5 & OOM &52.2$\pm$3.5 & 46.3$\pm$0.5 & 38.3$\pm$1.5 \\        HyperSAGE~\cite{arya2020hypersage} & message-passing &  72.4\p1.6 & 77.4\p3.8 & 69.3\p2.7 &72.9\p1.3 & 61.8\p2.3 \\
             UniGNN~\citep{UniGNN} &message-passing & 75.3\p1.2 & 88.8\p0.2 & 70.1\p1.4 & 74.4\p 1.0& 63.6\p 1.3 \\
        
        \midrule
              SHKC (ours) & kernel-based &\textbf{76.05\p 0.7(6)} & \textbf{89.17\p0.2(16)}  & \textbf{70.64\p 1.8 (32)} & \textbf{75.08\p 1.1(4)} &\textbf{65.14\p 1.0(32)}\\
        \bottomrule
    \end{tabular}
    }
    \label{tab:sota_accuray}    
\end{table*}

\paragraph{\textbf{Comparison with SOTAs.}}
As shown in Table \ref{tab:sota_accuray}, the results successfully verify the effectiveness of SHKC which achieves a new SOTA performance across all five datasets. 
We have the observations: (i) SHKC consistently outperforms the baselines, indicating that it can utilize the elaborately designed diffusion process. (ii) HGNN, HNHN and HGAT show poor performance on disconnected datasets(e.g. Citeseer), mainly due to the values of the row corresponding to an isolated vertex leading to information loss which corresponds to the theoretical explanation in Section $\ref{sec:shkc}$ that $l_1$-norm of the transition matrix in Eq.(\ref{eq:T}) can not be bounded. By modifying Eq.(\ref{eq:T}) to Eq.(\ref{eq:tansi}), SHKC makes up the flaws of performance degradation when meets isolated vertex.(iii) SHKC has a lower bias and standard deviation than others, showing better generalization. Furthermore, comparison of running time and computational complexity with existing hypergraph neural networks can be found in Appendix \ref{appsec:compute_complexity}.

\subsection{Visual Object Classification}
\vspace{-3mm}
\label{sub:object_classification}
\begin{table}[h]
\def\p{$\pm$} 
\centering
\setlength\tabcolsep{12pt} 
\caption{Test accuracy on visual object classification. GVCNN+MVCNN represents combining the features or structures to generate multi-modal data.}
\scalebox{0.8}{
\begin{tabular}{c|cc|ccccc}
    \toprule 
    \multirow{1}{*} { Datasets }& \multirow{1}{*} { Feature } & Structure  & HGNN & UniGNN & HGAT &  SHKC(ours) \\
    \hline
    \multirow{3}{*}{\begin{tabular}[c]{@{}l@{}}NTU\end{tabular}} 
    & MVCNN &  MVCNN & 80.11\p0.38  & 75.25$\pm$0.17 & 80.40$\pm$0.47  &  \textbf{82.56\p 0.39}  \\
    & GVCNN &  GVCNN  & 84.26\p0.30  & 84.63$\pm$ 0.21 & 84.45$\pm$0.12 &   83.35\p 0.30  \\
    & BOTH &  BOTH &  83.54\p0.50  & 84.45$\pm$0.40  & 84.05$\pm$0.36   &   \textbf{85.12\p0.25}  \\
    \hline
    \multirow{3}{*}{\begin{tabular}[c]{@{}l@{}}Model-\\Net40\end{tabular}} 
    & MVCNN       &  MVCNN       & 91.28\p 0.11  & 90.36$\pm$0.10 & 91.29$\pm$0.15             &  \textbf{92.01\p 0.08}  \\
    & GVCNN       &  GVCNN       &  92.53\p0.06  & \textbf{92.88$\pm$0.10} & 92.44$\pm$0.11 & 
    {92.69\p 0.06}    \\
    & BOTH& BOTH  &   {97.15\p0.14}  & 96.69$\pm$0.07 & 96.44$\pm$0.15  &  \textbf{97.78\p0.03}    \\
    \bottomrule
\end{tabular}
}
    \label{tab:object_accuracy} 
\end{table}
This experiment is about semi-supervised learning.
We employ two public benchmarks: Princeton ModelNet40
dataset~\citep{wu20153d} and the National Taiwan University~(NTU) 3D model dataset~\citep{chen2003visual} to evaluate our method. We follow HGNN~\citep{hgnn} to preprocess the data by MVCNN \citep{su2015multi} and GVCNN~\citep{feng2018gvcnn}.  Finally, we use the datasets provided by the public Code~\footnote{\url{https://github.com/iMoonLab/HGNN}}.
Details can be found in Appendix \ref{app:sub:object_classification}.

\textbf{Results.} 
Table \ref{tab:ModelNet40_accuracy} depicts that SHKC significantly outperform the image-input or point-input methods. These results demonstrate that SHKC can capture the similarity between objects in the hidden representation space to improve the performance of the classification task. 
Table \ref{tab:object_accuracy} compares our methods with HGNN on NTU and ModelNet40. 
\begin{wraptable}{r}{0.5\linewidth}
\def\p{$\pm$}
\centering
\setlength\tabcolsep{12 pt} 
\caption{ Classification accuracy  (\%) on  ModelNet40. The \textit{embedding} means the output representations of MVCNN+GVCNN Extractor.}
\scalebox{0.70}{
\begin{tabular}{l|cc}
    \toprule 
    \multirow{1}{*} { Methods } & \multirow{1}{*} { input }  & Accuracy \\
    \midrule
    MVCNN \citep{feng2018gvcnn} &  image  &  90.1  \\
    PointNet \citep{qi2017pointnet} &  point &  89.2  \\
    PointNet++ \citep{qi2017pointnet++} &  point &  90.1  \\
    DGCNN~\citep{wang2019dynamic} &  point &  92.2  \\
    InterpCNN~\citep{mao2019interpolated} & point & 93.0 \\
    SimpleView~\citep{uy2019revisiting} & image & 93.6 \\
    pAConv~\citep{xu2021paconv} & point & 93.9 \\
    \hline
    HGAT~\citep{HyperGAT} & embedding & 96.4 \\
    UniGNN~\citep{UniGNN} & embedding & 96.7 \\
    HGNN~\cite{hgnn} &  embedding &  97.2 \\
    \hline
    SHKC(ours) &    embedding  &  \textbf{ 97.7} \\
    \bottomrule
\end{tabular}
}
 \vspace{-15mm}
\label{tab:ModelNet40_accuracy}    
\end{wraptable}
From those results, we can see that our methods outperform HGNN on both single modality and multi-modality~(BOTH) datasets and our SHKC achieves much better performance on multi-modality compared  with others. These results reveal that our SHKC has the advantage of combining such multi-modal information through concatenating the  weighted incidence matrices~($\Q$) of hypergraphs, which means merging the multi-level hyperedges. 

\subsection{Over-Smoothing Analysis}\label{sec:Over-Smoothing_analysis}
It is worth to note that a stacked $k$-layers convolution could capture information from $k$-hop neighborhood. In order to capture long-range information, HGNN and HyperGCN are required to stack multi convolution layers leading to performance descending, as shown in Table \ref{tab:depth_accuracy}, which is well-known as over-smoothing issues. However, our proposed SHKC significantly avoids the performance descending by enlarge step of diffusion to capture long-range information rather than stacking multi convolution layers. In other words, SHKC can gain $k$-hop information by setting the diffusion step to $k$ rather than stacking $k$ one-step SHKC layers.    From \ref{tab:depth_accuracy}, when the number of layers in other models is same with diffusion steps in SHKC, it can be observed that SHKC outperforms the other models in almost all datasets, especially when number of layers is large. 
\begin{table}[htbp]
\setlength\tabcolsep{12pt}
\def\p{$\pm$} 
\centering
\caption{ Summary of classification accuracy (\%) results with various depths.  In our SHKC, the number of layers is equivalent to $t$ in Eq.~\eqref{eq:shkc}. We report mean test accuracy over 10 train-test splits. }
\scalebox{0.77}{
\begin{tabular} {p{1.8cm}p{2.4cm}|p{0.8cm}p{0.8cm}p{0.8cm}p{0.8cm}p{0.8cm}p{0.8cm}}
\toprule
\multirow{2}{*}{Dataset} & \multirow{2}{*}{Method} & \multicolumn{6}{c}{Layers/diffusion steps}  \\
& & 2     & 4     & 8     & 16    & 32    & 64       \\

\midrule
\multirow{3}{*}{\begin{tabular}[c]{@{}c@{}}Cora~\\(co-authorship)\end{tabular}} 
& HyperGCN\Done  & {60.66} & 57.50 & 31.09  &  31.10  & 30.09   &  31.09 \\
& HGNN\Done &{69.23} & 67.23 & 60.17 & 29.28 & 27.15 &  26.62    \\
& SHKC (ours)\Done & 74.60 & \textbf{75.78} & \textbf{75.70} & \textbf{75.04} & \textbf{75.26} & \textbf{74.79} \\
\midrule
\multirow{3}{*}{\begin{tabular}[c]{@{}l@{}}DBLP~\\(co-authorship)\end{tabular}}   
& HyperGCN\Done & {84.82} & 54.65 & 22.37 & 23.96 &  23.04  & 24.13      \\
& HGNN\Done & {88.55}& 88.28 & 85.38 & 27.64 &27.62 & 27.56 \\
& SHKC (ours)\Done & 86.63 & 88.26 & \textbf{89.00}& \textbf{89.17} & \textbf{89.05} &  \textbf{88.60} \\
\midrule
\multirow{3}{*}{\begin{tabular}[c]{@{}l@{}}Cora~\\(co-citation)\end{tabular}}     
& HyperGCN\Done & 62.35 & 58.29 & 31.09  & 31.17 & 31.09  & 29.68  \\
& HGNN\Done & {55.60}& 55.72& 42.10 & 26.16 & 24.40 & 24.43    \\
& SHKC (ours)\Done & 62.21 & 64.57 & \textbf{67.59} & \textbf{68.96} & \textbf{69.37} & \textbf{68.15} \\
\midrule
\multirow{3}{*}{\begin{tabular}[c]{@{}l@{}}Pubmed~\\(co-citation)\end{tabular}}   
& HyperGCN\Done & 68.12 & 63.59 & 39.99 & 39.97  & 40.01  & 40.02   \\
& HGNN\Done & {46.41} & 47.16 & 40.93 & 40.24 & 40.30 & 40.29    \\
& SHKC (ours)\Done & 74.39 & \textbf{74.91} & \textbf{74.41} & \textbf{73.90} & \textbf{72.79} & \textbf{71.49}  \\
\midrule
\multirow{3}{*}{\begin{tabular}[c]{@{}l@{}}Citeseer~\\(co-citation)\end{tabular}} 
& HyperGCN\Done & 56.94  & 36.75 & 20.72 & 20.41  & 20.16  & 18.95     \\
& HGNN\Done & {39.93 }& 38.98 & 36.67 & 19.91 &  19.86 & 19.79    \\
& SHKC (ours)\Done & 61.63 & \textbf{62.75} & \textbf{63.86} & \textbf{64.62} & \textbf{65.14} & \textbf{65.10}    \\
\bottomrule

\end{tabular}
}
\label{tab:depth_accuracy}
\end{table}

\vspace{-0.3cm}
\section{Conclusion}
In this paper, we review the design of convolution for discrete structure from a kernel perspective. We propose a two-phase paradigm 
that play roles in topology information aggregation and channel aggregation respectively to conduct convolutions. Specifically, we concentrate on hypergraph which is considered as a more general discrete structure to capture complex topology information. The proposed SHKC could adjust diffusion step to aggregate long-range information which avoids stacking multi existing convolution layers which leads to oversmoothing issues. Analysis based on uniform stability theorem corresponds to the outperforming empirical results on downstream tasks. 
\section{Acknowledgement}
This work was supported in part by the National Natural Science Foundation of China (61972219), the Research and Development Program of Shenzhen (JCYJ20190813174403598), the Overseas Research Cooperation Fund of Tsinghua Shenzhen International Graduate School (HW2021013).
\bibliographystyle{nips_bib}
\bibliography{reference}

\newpage
\appendix
 \section{Details of experiments}
 
 All the settings of our experiments follow \cite{ourspaper}. The details show below.
 
 \subsection{Citation Network Classification}
\paragraph{Datasets.}
The datasets we use for citation network classification include co-authorship and co-citation datasets: PubMed, Citeseer, Cora~\citep{sen2008collective} and DBLP~\citep{rossi2015network}. We adopt the hypergraph version of those datasets directly from  \citet{hypergcn}, where
hypergraphs are created on these datasets by assigning each document as a node and each hyperedge
represents (a) all documents co-authored by an author in the co-authorship dataset and (b) all documents
cited together by a document in co-citation dataset.  The initial features of each document (vertex) are represented by bag-of-words features. The details about vertices, hyperedges and features are shown in Table~\ref{tab:dataset_hypergcn}.

\begin{table}[htbp]
\centering
\caption{Real-world hypergraph datasets used in our citation network classification task.}
\renewcommand\tabcolsep{5.0pt} 
\label{tab:dataset_hypergcn}
\vspace{-0mm}
\scalebox{0.85}{
\begin{tabular}{l|c|c|c|c|c}
\toprule
Dataset& \# vertices & \# Hyperedges & \# Features & \# Classes &  \# isolated vertices\\
\midrule
\textbf{Cora}~(co-authorship)
& 2708         & 1072         &1433   & 7     & 320(11.8\%)   \\

\textbf{DBLP}~(co-authorship)
& 43413        & 22535        & 1425       & 6     & 0 (0.0\%)     \\

\textbf{Pubmed}~(co-citation)
& 19717        & 7963         & 500        & 3      & 15877 (80.5\%)   \\

\textbf{Cora}~(co-citation)
& 2708         & 1579         & 1433       & 7     & 1274 (47.0\%)     \\

\textbf{Citeseer}~(co-citation)
& 3312         & 1079         & 3703       & 6     & 1854 (55.9\%)     \\
\bottomrule
\end{tabular}
}
\end{table}

\paragraph{Settings and baselines.}

 We adopt the same dataset and train-test splits (10 splits) as provided in their publically available implementation\footnote{https://github.com/malllabiisc/HyperGCN, Apache License}. Note that this dataset just has the edge-independent vertex weights $\H$, which is also called the incidence matrix. So this experiment can be regarded as a special case of the specific application of our model~(i.e. $\Q=\H$).  
 
 For baselines MLP+HLR, HNHN~\citep{dong2020hnhn}~, HyperSAGE~\citep{arya2020hypersage}, UniGNN~\citep{UniGNN}, HGNN~\citep{hgnn} and FastHyperGCN~\citep{hypergcn}, HyperGCN~\citep{hypergcn}, UniGNN~\citep{UniGNN} we reuse the results reported by \citet{UniGNN}.
 For HNHN~\citep{dong2020hnhn} and HGAT~\citep{HyperGAT}, we implement them according to their public code.
 
We use cross-entropy loss and Adam SGD optimizer with early stopping with the patience of 100 epochs to train SHSC.
 For hyper-parameters, we use the grid search strategy.
 More details of hyper-parameters can be found in Table \ref{tab:hyper-parameter-citation_object}. 
 
\paragraph{Running Time and Computational Complexity}\label{appsec:compute_complexity}
Firstly, we analyze the theoretical computational complexity of SHKC:
For SHKC, the computational cost is the $\mathcal{O}(K|E|d+K|\mathcal V|d)$, which includes $K$ sparse matrix multiplication and $K$ summation over filters($|\mathcal V|d$  is the cost of adding features $X$).
Then, we compare the running time with existing models in Table \ref{tab:runningTime}. The results illustrate that our method is of the same order of magnitude as SOTA's approach UniGNN and outperforms the HyperGCN and HGAT.

\begin{table}[htbp]

\centering
\def\p{$\pm$} 
\vspace{-.2cm}
\caption{The average training time per epoch with different methods on citation network classification task is shown below and timings are measured in seconds.(OOM: Out of Memory)}
\scalebox{0.8}{

    \begin{tabular}{c|c|c|c|c|c}
    \toprule
     \text{Methods}  & \text{cora coauthorship} & \text{dblp coauthorship} & \text{cora cocitation} & \text{pubmed cocitation} & \text{citeseer cocitation} \\
    \midrule
    \text{HyperGCN} & 0.150\p0.058 & 1.181\p0.071 & 0.151\p0.029& 1.203\p0.104	& 0.130\p0.029 \\
    \text{HGNN} & 0.005\p0.002 & 0.081\p0.006 & 0.005\p0.040 & 0.008\p0.002 & 	0.005\p0.002\\
    \text{UniGNN } & 0.014\p0.044 & 0.042\p0.040 & 0.014\p0.042 & 0.023\p0.043 & 0.0168\p0.043\\
    \text{HNHN} & 0.001\p0.0026 & 0.007\p0.014 & 0.0010\p0.004 & 0.009\p0.006 &  0.001\p0.003 \\
    HGAT & 0.381\p0.080 & OOM & 0.279\p0.083 & 1.329\p0.016 &0.286\p0.087 \\
    \midrule
    SHKC (ours) & 0.055\p0.001 & 0.291\p0.001 & 0.205\p0.003 & 0.135\p0.056 & 0.193\p0.057 \\
    \bottomrule
    \end{tabular}
}
 \vspace{-.5cm}
\label{tab:runningTime}
\end{table}

 \subsection{Visual Object Classification} \label{app:sub:object_classification}
\begin{table}[htbp]
\centering
\vspace{-3mm}
\caption{summary of the ModelNet40 and NTU datasets}

\scalebox{0.8}{
    \begin{tabular}{c|c|c}
    \toprule
    \text { Dataset } & \text { ModelNet40 } & \text { NTU } \\
    \midrule
    \text { Objects } & 12311 & 2012 \\
    \text { MVCNN Feature } & 4096 & 4096 \\
    \text { GVCNN Feature } & 2048 & 2048 \\
    \text { Training node } & 9843 & 1639 \\
    \text { Testing node } & 2468 & 373 \\
    \text { Classes } & 40 & 67 \\
    \bottomrule
    \end{tabular}
}
\vspace{-2mm}
\label{tab:object_dataset}
\end{table}

\paragraph{Datasets and Settings.}
We employ two public benchmarks:  Princeton ModelNet40
dataset~\citep{wu20153d} and the National Taiwan University
(NTU) 3D model dataset~\citep{chen2003visual}, as shown in Table \ref{tab:object_dataset}.

In this experiment, each 3D object is represented by the feature vectors which are extracted by Multi-view Convolutional Neural Network (MVCNN) \citep{su2015multi} or Group-View Convolutional Neural Network (GVCNN) \citep{feng2018gvcnn}. The features generated by different methods can be considered as multi-modality features. The hypergraph structure we designed is similar to \citet{zhang2018dynamic}(
but they did not give spectral guarantees for supporting the rationality of their practices). We represent the hypergraph structure as an edge-dependent vertex weight $\Q$ (i.e. $\Q_1=\Q_2=\Q$). Specifically, we firstly generate hyperedges by $k$\textit{-NN} approach, i.e. each time one object can be selected as a centroid and its $k$ nearest neighbors are used to generate one hyperedge including the centroid itself~(in our experiment, we set $k=6$). Then, given the features of the data, the  vertex-weight matrix $\mathbf{Q}$ is defined as 
\begin{align}
    \Q(v,e)=\left\{\begin{array}{cl}
	\operatorname{exp}(\frac{-d(v,v_c)}{\gamma\hat{d}^2}), &\text{if } v\in e  \\
	0, & \text{otherwise},
\end{array}\right.
\end{align}

where $d(v,v_c)$ is the euclidean distance of features between an object $v$ and the centroid object $v_c$ in the hyperedge and $\hat{d}$ is the average distance between objects. $\gamma$ is a hyper-parameter to control the flatness.
As we have two-modality features generated by MVCNN and GVCNN, we can obtain the matrix $\Q_{\{i\}}$ which corresponds to the data of the $i$-th modality~($i\in\{1,2\}$). After all the hypergraphs from different features have been generated, these matrices $\Q_{\{i\}}$ can be concatenated to build the multi-modality hypergraph matrix $\Q=[\Q_{\{1\}},\Q_{\{2\}}]$. The features generated by GVCNN or MVCNN can be singly used, or concatenated to a multi-modal feature for constructing the hypergraphs. 
We use cross-entropy loss and Adam SGD optimizer with early stopping with the patience of 100 epochs to train SHSC. 
More details of hyper-parameters can be found in Table \ref{tab:hyper-parameter-citation_object}. 
\begin{table}[htbp] 
\centering
\caption{Hyper-parameter search range for citation network classification and visual object classification.}
\vspace{1mm}
\resizebox{0.8 \linewidth}{!}{
\begin{tabular}{c|c|c}
\toprule
     Methods & Hyper-parameter        & Range               \\ 
\midrule
    \multirow{10}{*}{\begin{tabular}[c]{@{}l@{}}SHKC\end{tabular}}
    & $\sigma$ & \{-2,-1,-0.5,0,0.5,1,2\} \\
    & $\gamma$(visual object classification)    & \{0.1, 0.2, 0.4, 0.5, 0.8,1.0\} \\
    & Learning rate        & \{0.001, 0.005, 0.01\}                  \\
    & Hidden dimension & \{128\} \\
    & Layers & \{2,4,6,8,16,32,64\} \\
    & Weight decay & \{1e-3,1e-4,5e-4, 1e-5\}  \\
    & $\alpha$ & \{1,0.97,0.95,0.9,0.85,0.8,0.75,0.7,0.65,0.6\}  \\
    & $\beta$ & \{ 1,0.95,0.90,0.85,0.8 \}  \\
    & Optimizer & Adam  \\
    & Epoch & 1000 \\
    & Early stopping patience & 100 \\
    & GPU & Tesla V100 \\
    
\bottomrule                     
\end{tabular}                           
}
\label{tab:hyper-parameter-citation_object}
\end{table}

\begin{figure*}[htbp]

\centering
    \subfigure{
    \begin{minipage}[s]{0.3\linewidth}
    \centering
    \includegraphics[width=1\linewidth]{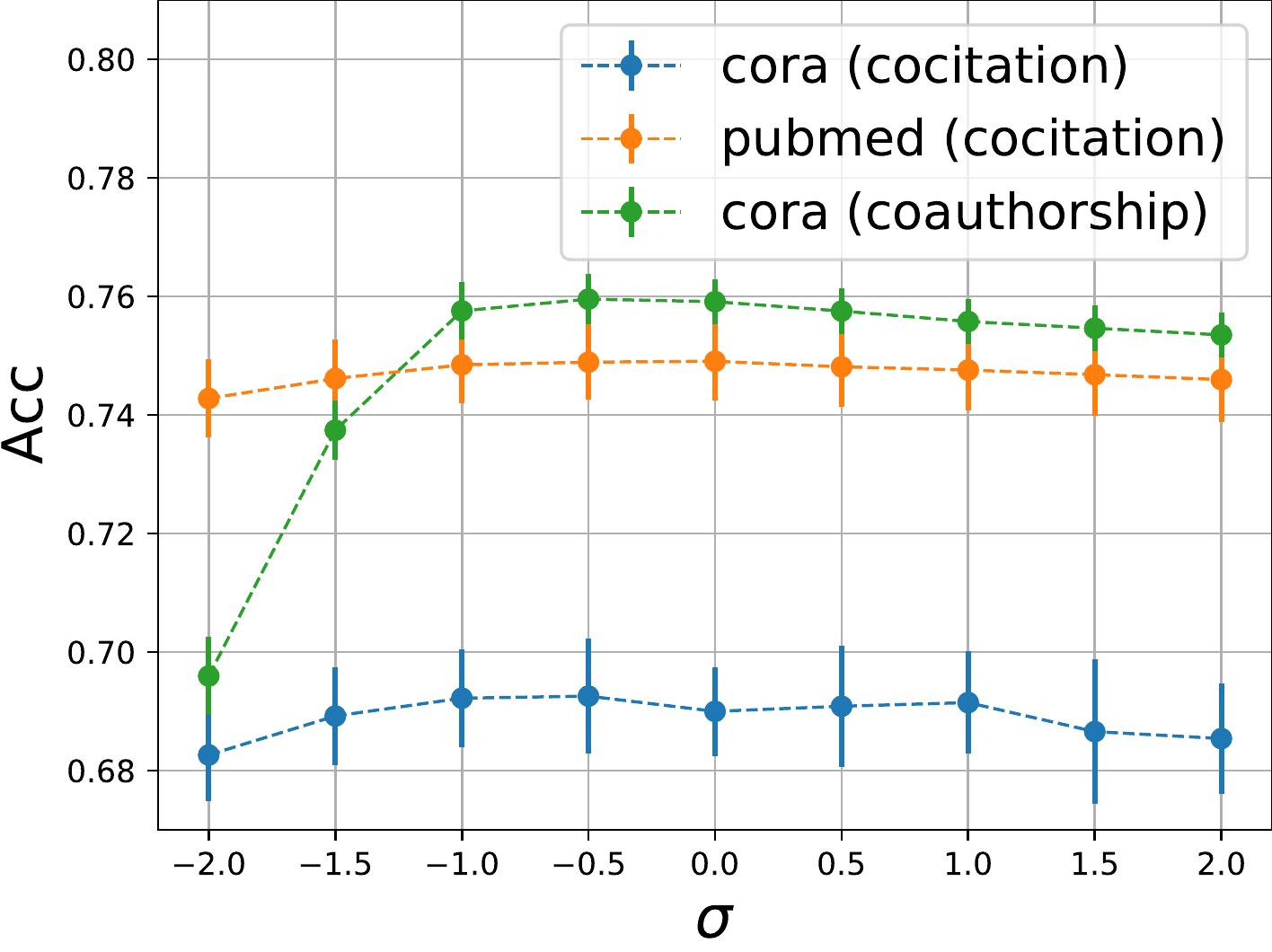}
    \end{minipage}%
    }%
    \quad 
    \subfigure{
    \begin{minipage}[s]{0.3\linewidth}
    \centering
    \includegraphics[width=1\linewidth]{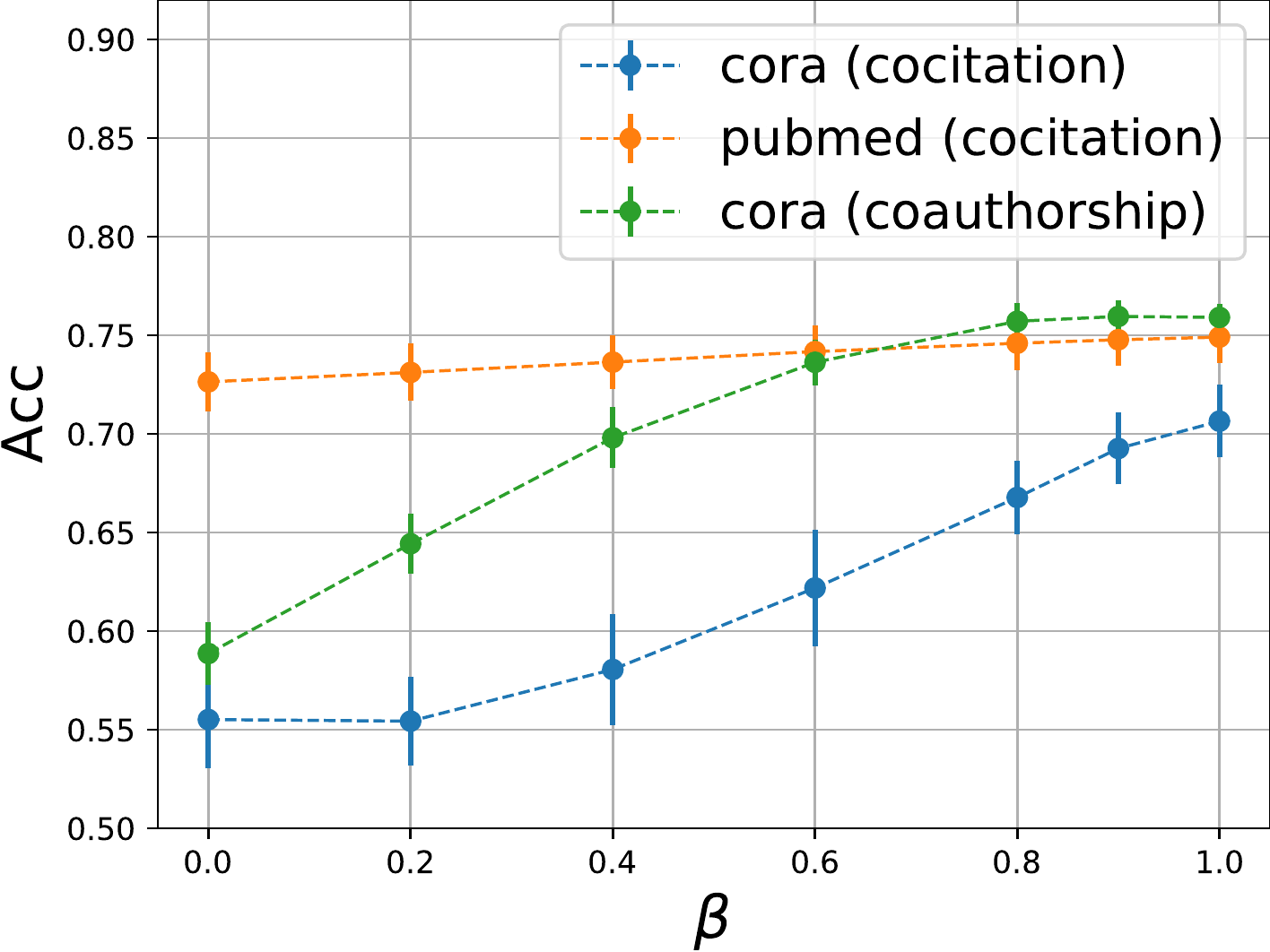}
    \end{minipage}%
    }%
    \quad
    \subfigure{
    \begin{minipage}[s]{0.3\linewidth}
    \centering
    \includegraphics[width=1\linewidth]{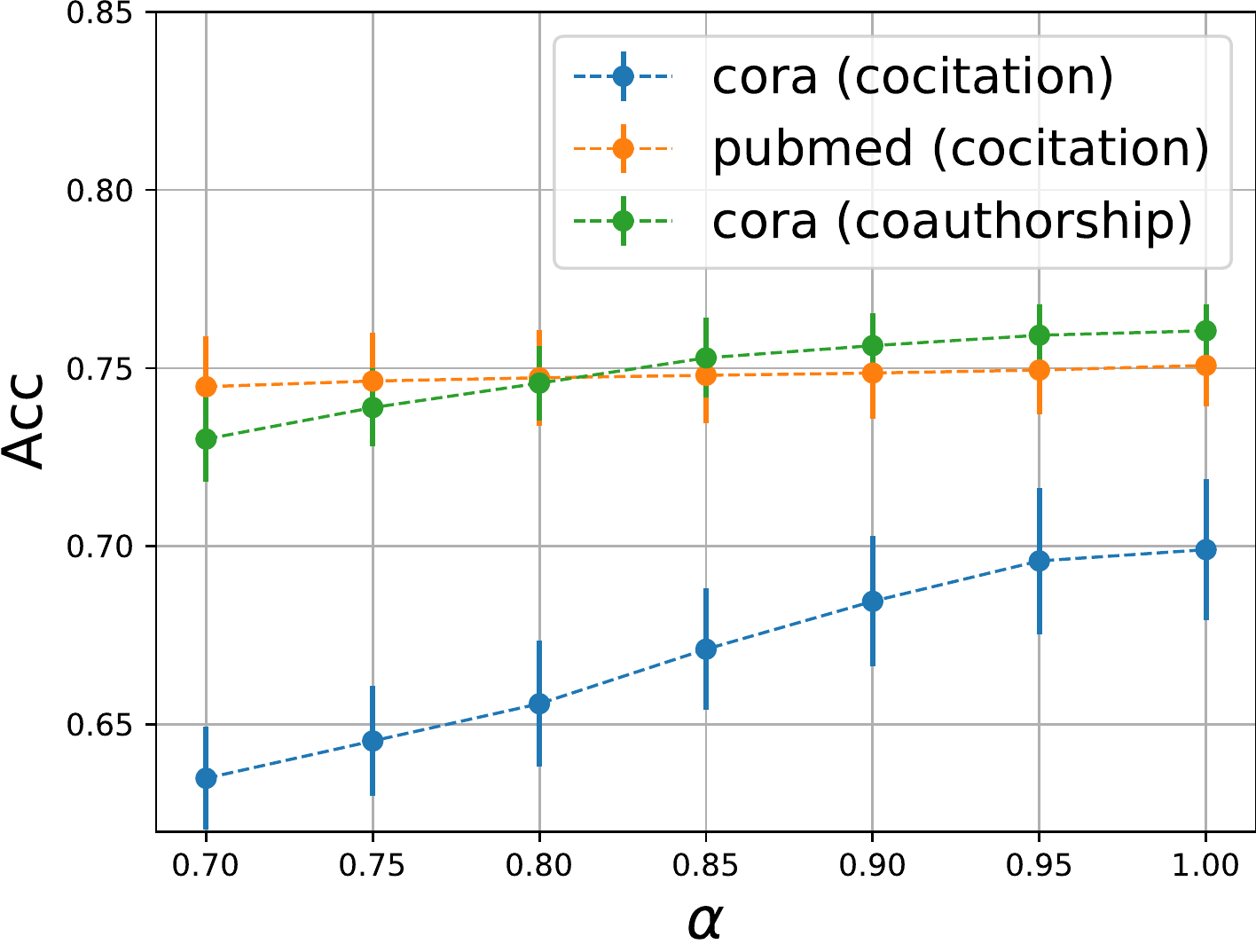}
    \end{minipage}
    }%
\centering
\caption{Test accuracy by varying the hyper-parameters $\sigma$~(left),  $\beta$~(middle) and $\alpha$~(right). }
\label{fig:sensitivity}
\end{figure*}

\subsection{Sensitivity Analysis}\label{sec:Sensitivity_analysis}
 Performance of SHKC on co-authorship cora, co-citation cora and co-citation pubmed with different $\rho$, $\beta$ and $\alpha$ is reported in Figure \ref{fig:sensitivity} where $\rho(\cdot) = (\cdot)^{\sigma}$. For $\sigma$, we can see that the best choice will vary depending on the dataset but mainly concentrate around $-0.5$, which verifies that the effect of hyperedges degree to the transition probability is various and it has negative effect in most cases. For $\beta$, with the growth of $\beta$, the performance of SHKC stably increases, which means that the topology information is successfully aggregated through the defined diffusion process. Moreover, the fact that the performance remains stable on cora~(coauthorship) and pubmed~(cocitation) when $\beta$ is at 0.8-1.0 suggests we can only adjust this hyper-parameter at a range of value close to 1. For $\alpha$, the tendency of it is similar to $\beta$, so we can also adjust this hyper-parameter at a range of value close to 1 to obtain a satisfying performance.

\section{Spectral view of SHKC}
\label{app:spectral_ana}
\paragraph{\textbf{Spectral View of SHKC.}}

Actually, SHKC is a spatial-based model while we can analyze it from a spectral-based view.
Let $\L = \I - \Tilde{\mbf T}$ denote the normalized Laplacian matrix. 
Then the SHKC can be viewed as a special spectral-based polynomial filter through $\sum_{i=0}^{I}\xi_{i}\L^{i} = \sum_{k=0}^{K}\theta_{k}\Tilde{\mbf T}^{k}$ where $\theta_{k}=\frac{\alpha^{k}}{K}$. Finally, we can deduce the coefficients of the special polynomial filter as $\xi_{i} =(-1)^{i}\sum_{k=i}^{K}\binom{k}{i}\frac{\alpha^{k}}{K}$ showing the strong relationships between SHKC and spectral-based model.
The explicit form of $\xi_{i}$ can be derived as:
\begin{align*}
    \sum_{k=0}^{K}\theta_{k}\Tilde{\mbf T}^{k} &= \sum_{k=0}^{K}\theta_{k}(\I-\L)^{k} = \sum_{k=0}^{K}\theta_{k}\sum_{i=0}^{k}\binom{k}{i}(-1)^{i}\L^{i}\\
    &= \sum_{i=0}^{K}\left(\sum_{k=i}^{K}(-1)^{i}\binom{k}{i}\frac{\alpha^{k}}{K}\right)\L^{i}
\end{align*}
\paragraph{The Spectrum Analysis of SHKC}
We calculate the eigenvalues of $\left(\beta\sum_{k=1}^{K}\frac{\alpha^{k}}{K} \tilde{\mathbf{T}}^{k} + (1-\beta)\mathbf{I}\right)$  with various $(\alpha,\beta)$ on NTU2012 dataset ($|\mathcal V|=2012$)  and count the number of eigenvalues in different size ranges, which is shown in Table \ref{tab:eigenvalues}. The table suggests that our SHKC can capture both low and high-frequency information of the graph signal, depending on the selection of the appropriate hyper-parameters $(\alpha,\beta)$.
\begin{table*}[ht]
    \centering
    \vspace{0cm} 
    \caption{The number of eigenvalues of SHKC in different size ranges on NTU2012 dataset.}
    \setlength\tabcolsep{4pt} 
    
    \scalebox{0.8}{
    \begin{tabular}{c|c c c c c c c c c c}
    \toprule
          $\alpha, \beta$ &$\lambda \geq 0.0$ & $\lambda \geq 0.1$ & $\lambda \geq 0.2$ & $\lambda \geq 0.3$ & $\lambda \geq 0.4$& $\lambda \geq 0.5$ & $\lambda \geq 0.6$ & $\lambda \geq 0.7$ & $\lambda \geq 0.8$ & $\lambda \geq 0.9$   \\
        \midrule
         (1,1)& 2012 & 364 & 163 & 108 & 82 & 66 & 52 & 38 & 26 & 13 \\
        (1,0.8)& 2012 & 327 & 140 & 90 & 71 & 51 & 33 & 19 & 1 & 0 \\
        (0.8,1)& 2012 & 161 & 40 & 0 & 0 & 0 & 0 & 0 & 0 & 0 \\
        (0.8,0.8)&2012 & 139 & 8 & 0 & 0 & 0 & 0 & 0 & 0 & 0 \\
        \bottomrule
    \end{tabular}
    }
     \vspace{-0.3cm}
    \label{tab:eigenvalues}
\end{table*}

\section{Main Lemma.}\label{app:main_lemma}

\begin{restatable}{lemma}{restartlmone}
\label{lm:hmax}
Let $h_{max}^{(l)} = max_i \Vert [\mathbf{H}^{(l)}]_{i,:}\Vert_2$ be the maximum norm of the vertex hidden representation with an $l$-step diffusion. Then for $L$, $h_{max}^{(L)}$ is bounded by:
\vspace{-0.2cm}
\begin{align*}
\small
    h_{max}^{(L)} \leq (\frac{\beta}{L}\sum_{l=1}^{L}(\alpha \sqrt{1+\rho_{max}ED})^l + 1-\beta)C_x C_{\Theta} 
\end{align*}
\vspace{-0.2cm}
\end{restatable}

The proof is referred to in Appendix \ref{proof:l1}.
Lemma \ref{lm:hmax} derives the bound of the maximum norm of the hidden representation among all vertices which is required in the proofs of Lemma \ref{lm:gra} and Theorem \ref{th:main}.

\begin{restatable}{lemma}{restartlmtwo}
\label{lm:dhmax}
Define $\Delta h_{max}^{(l)} = \max_i \Vert[\mbf H^{(l)}(\mbf \Theta) - \mbf H^{(l)}(\Tilde{\mbf \Theta})]_{i,:} \Vert_2$ as the maximum distance of  hidden vertex representation with an $l$-step diffusion   between different parameters $\mbf \Theta$ and $\Tilde{\mbf \Theta}$. Then, this distance with an $l$-step diffusion can be bounded by:
 \vspace{-0.5cm}
\begin{align*}
\small
    \Delta h_{max}^{(L)} \leq (\frac{\beta}{L}\sum_{l=1}^{L}(\alpha d_T)^l + 1-\beta)C_x\Vert\Delta\mbf\Theta \Vert_2 
\end{align*}
where $\Delta\mbf\Theta = \mbf\Theta -\Tilde{\mbf\Theta} $.
\end{restatable}
The proof is referred to Appendix \ref{proof:l2}. 
This Lemma bounds the maximum norm of changes of hidden vertex representation with $L$-step diffusion between two different sets of parameters which is significant to compute the Lipschitz $L_{\mathcal{M}}$ and smoothness constant $S_{\mathcal{M}}$ of SHKC.

\begin{restatable}{lemma}{restartlmthree}
\label{lm:gra}
Define $\mbf G_{\mbf\Theta} =\frac{\partial f }{\partial\mbf\Theta}$ as the gradient of the prediction concerning learnable weights $\mbf \Theta$ and $\mbf G_{\vec{\omega}} =\frac{\partial f }{\partial\vec{\omega}}$ as the gradient of the prediction concerning the weights of the classifier. Let $\Delta_{\mbf \Theta} \mbf G = \mbf G_{\mbf\Theta} - \mbf G_{\Tilde{\mbf\Theta}}$ and $\Delta_{\vec{\omega}} \mbf G = \mbf G_{\vec{\omega}} - \mbf G_{\Tilde{\vec{\omega}}}$ be the difference between gradient computed on two different learnable and classifier weights.
Then we have:
\begin{align*}
&\Vert\mbf G_{\mbf \Theta} \Vert_2 \leq  C_x C_{\alpha\beta L};\quad
\Vert\mbf G_{\vec{\omega}} \Vert_2 \leq  C_x C_{\alpha\beta L} C_{\Theta};\\
&\Vert\Delta_{\mbf \Theta}\mbf G \Vert_2 \leq  C_x^2 C_{\alpha\beta L}^2\Vert\Delta\mbf\Theta \Vert_2 +  C_x^2 C_{\alpha\beta L}^2 C_{\Theta}\Vert\Delta \vec{\omega}\Vert_2 +  C_x C_{\alpha\beta L}\Vert\Delta \vec{\omega}\Vert_2 ;\\
&\Vert\Delta_{\vec{\omega}}\mbf G \Vert_2 \leq  C_x^2 C_{\alpha\beta L}^2 C_{\Theta}^2 \Vert\Delta \vec{\omega} \Vert_2 +   C_x^2 C_{\alpha\beta L}^2 C_{\Theta} \Vert\Delta \mbf \Theta \Vert_2 +  C_x C_{\alpha\beta L} \Vert\Delta \mbf \Theta \Vert_2
\end{align*}
where $C_{\alpha\beta L} = \frac{\beta}{L}\sum_{l=1}^{L}(\alpha d_T)^l + 1-\beta$ and $d_T = \sqrt{1+\rho_{max}ED}$ in Proposition \ref{Th:Tl1}.
\end{restatable}
The proof is referred to Appendix \ref{proof:l3}. 

\section{Missing Proofs.}

\subsection{The Lemma 1}
\label{proof:l1}
\restartlmone*

\begin{proof}
For deduction simplicity, we use $\mbf T$ to represent the generalized transition matrix in the following deduction and define $d_T=\sqrt{1+\rho_{max}ED}$ from Proposition \ref{Th:Tl1}.
We first deduce the reclusive formulation of SHKC between layers as:
\begin{align}
\label{eq:reclu}
\mbf H^{(0)} &= \mbf X\mbf \Theta;\nonumber\\
\mbf H^{(l+1)} &= \frac{\alpha l}{l+1}\mbf T\mbf H^{(l)} + \frac{\alpha}{l+1}\mbf T\mbf H^{(0)}, ~ ~l=0,\cdots,L-2;\\
\mbf H^{(L)}&=\psi( \beta(\frac{\alpha(L-1)}{L}\mbf T\mbf H^{(L-1)} + \frac{\alpha}{L}\mbf T\mbf H^{(0)}) + (1-\beta)\mbf X\mbf \Theta) = \psi(\mbf Z^{(L)}) \label{eq:ZL}
\end{align}
where $\psi$ is the ReLu activation function.

From $h_{max}^{(L)}$, we deduce the bound as:
\begin{align}
\label{eq:hmax}
    h_{max}^{(L)} &= \max_{i} \Vert [\psi(\beta(\frac{\alpha(L-1)}{L}\mbf T\mbf H^{(L-1)} + \frac{\alpha}{L}\mbf T\mbf H^{(0)}) + (1-\beta)\mbf X\mbf \Theta)]_{i,:} \Vert_2 \nonumber\\
    &\leq \max_{i} \Vert\frac{[\beta\alpha(L-1)}{L}\mbf T\mbf H^{(L-1)}]_{i,:} \Vert_2 + \Vert[\frac{\beta\alpha}{L}\mbf T\mbf X\mbf\Theta]_{i,:} \Vert_2 +\Vert(1-\beta)[\mbf X\mbf\Theta]_{i,:} \Vert_2\nonumber\\
    &\leq_{(a)} \max_{i} \frac{\beta\alpha(L-1)}{L}\Vert \sum_j T_{ij}\mbf h_j^{(L-1)} \Vert_2 +\frac{\beta\alpha}{L}\Vert\sum_{j}T_{ij}\mbf h^{(0)}_j \Vert_2 +(1-\beta)C_x C_\Theta \nonumber\\
    &\leq  \frac{\beta\alpha(L-1)}{L}d_{T}h_{max}^{(L-1)} + \frac{\beta\alpha}{L}d_T C_x C_{\Theta} + (1-\beta)C_x C_{\Theta}
\end{align}
where inequality (a) is due to Proposition \ref{Th:Tl1} and assumptions in Theorem \ref{Th:gap}.
From the reclusive formulation of SHKC in Eq. \eqref{eq:reclu}, we have:
\begin{align}
\label{eq:h_l-1}
    h_{max}^{(L-1)} &\leq \frac{\alpha(L-2)}{L-1}d_Th_{max}^{(L-2)} + \frac{\alpha}{L-1}d_T C_x C_{\Theta}\nonumber\\
    & \cdots \nonumber\\
& \leq \frac{\sum_{l=1}^{L-1}(\alpha d_T)^l}{L-1}C_x C_{\Theta}
\end{align}
Add the Eq. \eqref{eq:h_l-1} to Eq. \eqref{eq:hmax} we get:
\begin{align*}
    h_{max}^{(L)} \leq (\frac{\beta}{L}\sum_{l=1}^{L}(\alpha d_T)^l + 1-\beta)C_x C_{\Theta}
\end{align*}
\end{proof}

\subsection{The Lemma 2}
\label{proof:l2}
\restartlmtwo*

\begin{proof}
We first denote $\mbf H^{(l)}(\Tilde{\mbf \Theta}) $ as $ \Tilde{\mbf H}^{(l)}$ for similarity in deduction below.
\begin{align*}
\Delta h_{max}^{(L)} &= \max_i \Vert[\mbf H^{(L)}(\mbf \Theta) - \mbf H^{(L)}(\Tilde{\mbf \Theta})]_{i,:} \Vert_2 \\
&\leq \max_i \Vert\frac{[\beta\alpha(L-1)}{L}\mbf T(\mbf H^{(L-1)} - \Tilde{\mbf H}^{(L-1)}) \\ 
&+  \frac{\beta\alpha}{L}\mbf T\mbf X(\mbf \Theta - \Tilde{\mbf\Theta}) + (1-\beta)\mbf X(\mbf \Theta - \Tilde{\mbf\Theta})]_{i,:} \Vert_2\\
&\cdots\\
&\leq (\frac{\beta}{L}\sum_{l=1}^{L}(\alpha d_T)^l + 1-\beta)C_x\Vert\Delta\mbf\Theta \Vert_2
\end{align*}
\end{proof}

\subsection{The Lemma 3}
\label{proof:l3}
\restartlmthree*

\begin{proof}
Note that the classifier $f$ has the formulation as  $f(\mbf h_i^{(L)}) =\sigma( \mbf h_i^{(L)}\vec{\omega})$ where $\sigma$ is the sigmoid function and $\mbf z_i^{(L)}$ is the $i$th node representation in the last layer before the activation function $\psi$.
\begin{align*}
\Vert\mbf G_{\mbf \Theta} \Vert_2  &= \frac{1}{m}\sum_i\Vert\frac{\partial Loss(f(\mbf h_i^{(L)}),y_i)}{\partial \mbf\Theta} \Vert_2\\
&\leq \max_i \kappa \sigma'( z_i)\Vert \frac{\partial \vec{\omega}^{\top}\mbf h_i^{(L)}}{\partial \mbf\Theta}\Vert_2\\
&\leq\max_i \kappa \Vert([\beta\sum_{l=1}^{L}\frac{(\alpha\mbf T)^l}{L}\mbf X\mbf + (1-\beta)\mbf X\mbf]_{i,:})^{\top}\mbf 1 Diag(\psi'(\mbf z_{i}^{(L)})\odot\vec{\omega}) \Vert_2\\
&\leq \max_i \kappa \Vert([\beta\sum_{l=1}^{L}\frac{(\alpha\mbf T)^l}{L}\mbf X\mbf + (1-\beta)\mbf X\mbf]_{i,:})^{\top}\mbf 1 \Vert_2 \Vert \vec{\omega} \Vert_2\\
&\leq \max_i \kappa \Vert[\beta\sum_{l=1}^{L}\frac{(\alpha\mbf T)^l}{L}\mbf X\mbf + (1-\beta)\mbf X\mbf]_{i,:}\Vert_2\\
&\leq \kappa C_x(\frac{\beta\sum_{l=1}^{L}(\alpha d_T)^l}{L} + 1-\beta)
\end{align*}

For the norm of $\mbf G_{\vec{\omega}}$:
\begin{align*}
    \Vert\mbf G_{\vec{\omega}} \Vert_2  &= \frac{1}{m}\sum_i\Vert\frac{\partial Loss(f(\mbf h_i^{(L)}),y_i)}{\partial \vec{\omega}} \Vert_2 \\
    &\leq \max_i \kappa\sigma'(z_i)\mbf h_{i}^{(L)} \\
    &\leq \kappa (\frac{\beta}{L}\sum_{l=1}^{L}(\alpha d_T)^l + 1-\beta)C_x C_{\Theta}
\end{align*}
For the norm of $\Delta_{\mbf \Theta}\mbf G$:
\begin{align*}
    \Vert \Delta_{\mbf \Theta}\mbf G \Vert_2 &\leq \max_i \kappa \Vert\frac{\partial \mbf h_i^{(L)}}{\partial\mbf \Theta }(\frac{\partial f(\mbf h_i^{(L)})}{\partial \mbf h_i^{(L)}} - \frac{\partial f(\Tilde{\mbf h}_i^{(L)})}{\partial \Tilde{\mbf h}_i^{(L)}} ) \Vert_2\\
&\leq\max_i\kappa \Vert[\beta\sum_{l=1}^{L}\frac{(\alpha\mbf T)^l}{L}\mbf X\mbf + (1-\beta)\mbf X\mbf]_{i,:})^{\top}\mbf 1 Diag(\mbf r -\Tilde{\mbf r}) \Vert_2\\
&\leq\kappa C_x C_{\alpha\beta L}( \Vert\Delta \vec{\omega}\Vert_2 (1+h_{max}^{(L)})+\Delta h_{max}^{(L)})\\
&\leq \kappa C_x^2 C_{\alpha\beta L}^2\Vert\Delta\mbf\Theta \Vert_2 + \kappa C_x^2 C_{\alpha\beta L}^2 C_{\Theta}\Vert\Delta \vec{\omega}\Vert_2 + \kappa C_x C_{\alpha\beta L}\Vert\Delta \vec{\omega}\Vert_2 
\end{align*}

where $\mbf r = \sigma'(\vec{\omega}^{\top}\mbf h_i^{(L)})\psi'(\mbf z_{i}^{(L)})\odot\vec{\omega}$ and $C_{\alpha\beta L} = \frac{\beta}{L}\sum_{l=1}^{L}(\alpha d_T)^l + 1-\beta $.

Finally, for $\Delta_{\vec{\omega}}\mbf G $:
\begin{align*}
    \Vert \Delta_{\vec{\omega}}\mbf G \Vert_2 &\leq \kappa (\Delta h_{max}^{(L)} + h_{max}^{(L)}(\Delta h_{max}^{(L)} + h_{max}^{(L)}\Vert\vec{\omega} \Vert_2)) \\
    &\leq \kappa C_x^2 C_{\alpha\beta L}^2 C_{\Theta}^2 \Vert\Delta \vec{\omega} \Vert_2 +  \kappa C_x^2 C_{\alpha\beta L}^2 C_{\Theta} \Vert\Delta \mbf \Theta \Vert_2 + \kappa C_x C_{\alpha\beta L} \Vert\Delta \mbf \Theta \Vert_2
\end{align*}

\end{proof}

\subsection{The Theorem \ref{th:main}}\label{th:main_proof}
\restartTheoremOne*
\begin{proof}
We use $\{\mbf\Theta,\vec{\omega}\}$ and $\{\Tilde{\mbf\Theta},\Tilde{\vec{\omega}}\}$ to denote two different sets of parameters throughout the model SHKC. 
We first give the deduction of $L_{\mcal M} = C_x C_{\alpha\beta L}\max\{1,C_{\Theta}\}$ from Lemma \ref{lm:hmax} and Lemma \ref{lm:dhmax}:
\begin{align*}
    \max_i \vert f(\mbf \Theta,\vec{\omega}|\mbf x_i) - f(\Tilde{\mbf \Theta},\Tilde{\vec{\omega}}|\mbf x_i) \vert &\leq \Delta h_{max}^{(L)} + h_{max}^{(L)}\Vert\Delta\vec{\omega} \Vert_2 \\
    &\leq C_x C_{\alpha\beta L}\Vert\Delta\mbf\Theta \Vert_2 + C_x C_{\alpha\beta L}C_{\Theta}\Vert \Delta\vec{\omega} \Vert_2\\
    &\leq C_x C_{\alpha\beta L}\max\{1,C_{\Theta}\}(\Vert\Delta\mbf\Theta\Vert_2+ \Vert\Delta\vec{\omega}\Vert_2)
\end{align*}
Then, from Lemma \ref{lm:gra}, we get the bound of the gradient $G_{\mcal M} = C_x C_{\alpha\beta L}(1+C_{\Theta})$
Finally, from Lemma \ref{lm:gra}, we have the smoothness of $\mcal M$ that $S_{\mcal M}$:
\begin{align*}
    &\max_i \Vert \nabla_{\mbf \Theta,\mbf \v}f(\mbf \Theta,\vec{\omega}|\mbf x_i) - \nabla_{\Tilde{\mbf \Theta},\Tilde{\mbf \v}}f(\Tilde{\mbf \Theta},\Tilde{\vec{\omega}}|\mbf x_i) \Vert_2 \leq \\
    &(\Vert\Delta\mbf \Theta \Vert_2 + \Vert\Delta\mbf\v \Vert_2)(C_x^2 C_{\alpha\beta L}^2\max\{1,C_{\Theta}\}^2 + C_x^2 C_{\alpha\beta L}^2 C_{\Theta} + C_x C_{\alpha\beta L})
\end{align*}

\end{proof}

\subsection{Proof of Proposition \ref{Th:Tl1}}
\label{proof:p1}
\begin{proof}
We have $\Tilde{\mbf T}$ to be formulated as:
\begin{align*}
\Tilde{\mbf T} = \Tilde{\mbf D}_v^{-1/2}(\mbf I + \mbf Q\mbf W\rho(\mbf D_{\mathcal{E}})  \mbf Q^\top )\Tilde{\mbf D}_v^{-1/2}
\end{align*}
Then the $l_1$ norm of $\Tilde{\mbf T}$ is bounded as:
\begin{align*}
    \Vert \Tilde{\mbf T} \Vert_1 &= \max_i \sum_{j} \vert\Tilde{T}_{ij}\vert= \max_i \sum_j \frac{\mbf 1\{i=j\} + \sum_e Q(i,e)W(e)\rho(\delta(e))Q(j,e)}{\sqrt{\Tilde{d}_v(i)}\sqrt{\Tilde{d}_v(j)}}\\
    &\leq_{(a)} \max_i\frac{\Tilde{d}_v(i)}{\sqrt{\Tilde{d}_v(i)}} = \max_i \sqrt{\Tilde{d}_v(i)}\\
    &\leq_{(b)} \sqrt{1+E\rho_{max}}
\end{align*}
where $\Tilde{d}_v(i) = 1+\sum_e Q(i,e)w(e)\rho(\delta(e))\delta(e)$ and $\delta(e) =\sum_v Q(v,e)$.
The inequality (a) is due to $\Tilde{d}_v(i)\geq 1$. The inequality (b) is due to $\delta(e) =\sum_v Q(v,e)\leq n_e\leq E$ and
\begin{align*}
\Tilde{d}_v(i) &= 1+\sum_e Q(i,e)w(e)\rho(\delta(e))\delta(e)\\
&\leq 1 + \sum_e \rho(\delta(e))\delta(e) \leq 1 + \rho_{max}ED   
\end{align*}

\end{proof}

\section{Uniform Stability Bound}
\label{app:usb}
We slightly modify the Uniform Stability Bound in \cite{el2006stable} for more distinct analysis of generalization error of SHKC.
\begin{theorem}[Uniform Stability Bound~\citep{el2006stable}]
\label{Th:gap}
Denote $\mathcal{M}$ to be a transductive learning model with uniform stability $\mu$. Let $K(m,n)=\sum_{i=1}^{m}\frac{n^2}{(n+i)^2}$. For a  $\kappa$-Lipschitz loss function bounded in $[0,1]$ and $\delta>0$, the gap between the training error and testing error is bounded as:
\begin{align*}
\vspace{-0.6cm}
\small
    gap(\mathcal{M},\gamma,\delta)&\leq \mu\kappa(1+\mathcal{O}(2\sqrt{2K(m,n)\ln{\delta^{-1}}}))\\ &+ \mathcal{O}(\frac{m+n}{mn}\sqrt{2K(m,n)\ln{\delta^{-1}}})
\end{align*}
where $m$,$n$ denote the number of samples in the training and testing sets respectively. 
\end{theorem}
\end{document}